\documentclass[12pt]{scrartcl}

\usepackage[utf8]{inputenc}
\usepackage{amsmath,amsthm,amssymb,
amsfonts,thmtools,,enumitem,graphicx,tikz}
\usepackage[noend]{algpseudocode}
\usepackage{hyperref} %

\usepackage{stmaryrd}
\SetSymbolFont{stmry}{bold}{U}{stmry}{m}{n}
\usepackage{todonotes}

\numberwithin{equation}{section}

\newtheorem{theorem}{Theorem}[section]
\newtheorem{corollary}[theorem]{Corollary}
\newtheorem{lemma}[theorem]{Lemma}

\declaretheorem[style=definition,qed=$\lrcorner$,numberwithin=section]{example}

\newcommand{\uend}{\hfill$\lrcorner$}

\newcounter{claimcounter}
\newenvironment{claim}[1][]{
  \renewcommand{\proof}{\smallskip\par\noindent\textit{Proof. }}
  \medskip\par\noindent%
  \ifthenelse{\equal{#1}{}}{%
    \setcounter{claimcounter}{0}\refstepcounter{claimcounter}\textit{Claim~\arabic{claimcounter}.}
  }{%
    \ifthenelse{\equal{#1}{resume}}{%
      \refstepcounter{claimcounter}\textit{Claim~\arabic{claimcounter}.}
    }{%
      \textit{Claim~#1.}
    }
  }
}{
  \par\medskip
}

\newcommand{\tp}{\operatorname{tp}}
\newcommand{\ltp}{\operatorname{ltp}}

\newcommand{\RR}{\mathbb R}
\newcommand{\NN}{\mathbb N}
\newcommand{\logic}[1]{\text{\upshape #1}}
\newcommand{\LL}{\logic L}
\newcommand{\FO}{\logic{FO}}
\newcommand{\smid}{\mathbin{;}}
\newcommand{\bigmid}{\mathbin{\big|}}
\renewcommand{\phi}{\varphi}
\newcommand{\err}{\operatorname{err}}
\newcommand{\CD}{{\mathcal D}}
\newcommand{\CC}{{\mathcal C}}
\newcommand{\CH}{{\mathcal H}}

\newcommand{\CT}{{\mathcal T}}
\newcommand{\BU}{{\mathbb U}}
\renewcommand{\theta}{\vartheta}
\newcommand{\ceil}[1]{\left\lceil#1\right\rceil}

\newcommand{\dist}{\operatorname{dist}}
\newcommand{\VC}{\operatorname{VC}}

\newcommand{\model}[4]{\llbracket#2(#3\smid#4)\rrbracket^B}

\begin{document}
\title{Learning first-order definable concepts over structures of small degree}

\author{\large Martin Grohe\\\normalsize RWTH Aachen
  University\\[-0.8ex]\normalsize grohe@informatik.rwth-aachen.de
\and \large Martin Ritzert\\\normalsize RWTH Aachen
  University\\[-0.8ex]\normalsize ritzert@informatik.rwth-aachen.de}
\date{}

\maketitle
\begin{abstract}
  We consider a declarative framework for machine learning where
  concepts and hypotheses are defined by formulas of a logic over some
  ``background structure''. We show that within this framework,
  concepts defined by first-order formulas over a background structure
  of at most polylogarithmic degree can be learned in polylogarithmic
  time in the ``probably approximately correct'' learning sense.
\end{abstract}

\section{Introduction}
This paper studies, from a theoretical perspective, a role that logic
might play as the foundation of a more declarative approach to machine
learning.  Machine learning algorithms produce a \emph{hypothesis} $H$
about some unknown \emph{target function} $C^*$ defined on an
\emph{instance space} ${\BU}$. In a supervised learning setting, the
input of a learning algorithm (the ``data'') consists of a sequence of
\emph{labelled examples}, that is, instances $u\in {\BU}$ labelled by
the value $C^*(u)$. The quality of the hypothesis $H$ is measured in
terms of how well it \emph{generalises}, that is, predicts correct
values of the target function on new data items.  In this paper, we
focus on \emph{Boolean classification} problems, where the target
function has the range $\{0,1\}$. In this case, we usually speak of a
\emph{target concept}. We also consider a setting where the target
concept is not deterministic, but a random variable.

The type of hypothesis we get is determined by the learning algorithm
we use. For example, if we use support vector machines, the hypothesis
is a linear halfspace of the instance space\footnote{We assume the
  instance space is $\RR^\ell$ for some $\ell$ and the hypothesis is a
  halfspace determined by a hyperplane.}, if we use decision tree
learning, then the hypothesis is a decision tree, and if we use deep
learning the hypothesis is specified by the weights and structure of a neural
network. The natural workflow would be to first decide on a
\emph{model} of how the target concept might look, or rather, what
kind of hypothesis might be appropriate. Then the learning algorithm
solves an optimisation problem by choosing the \emph{parameters} of
the model in such a way that they fit the data. For example, if
the instance space is $\RR^\ell$ and we choose a linear model, the
parameters of the model consist of a vector $\boldsymbol a\in\RR^\ell$
and a number $b\in\RR$, specifying the hyperplane $\{\boldsymbol
u\in\RR^\ell\mid \boldsymbol a\cdot\boldsymbol u-b\ge 0\}$. Then we
may choose an algorithm such as support vector
machine\footnote{Arguably, we could also call the support vector
  machine the ``model'' and the solver for the quadratic optimisation
  system behind it the ``algorithm''.} or the
perceptron algorithm for computing the parameters. %

From a declarative viewpoint, it seems desirable to separate the
choice of the model from the choice of the algorithm. Then as
logicians, we will ask which language we best use to describe the
model. A natural and very flexible framework to do this is the
following. We first choose a \emph{background structure} $B$. For
example, if we have numerical data, $B$ may be the the field of reals,
possibly expanded additional functions like the sigmoid function
$x\mapsto\frac{1}{1+e^{-x}}$. If we have graph data, our background
structure may be a finite labelled
graph. %
Given the background structure, we can specify a parametric model by a
formula $\phi(\bar x\smid\bar y)$ of some logic $\LL$, for example
first-order logic ($\FO$). This formula has two types of free
variables, the \emph{instance variables} $\bar x=(x_1,\ldots,x_k)$ and
the \emph{parameter variables} $\bar y=(y_1,\ldots,y_\ell)$. The
instance space ${\BU}$ of our model is $U(B)^k$, where $U(B)$ denotes
the universe of our background structure $B$. For each choice
$\bar v\in U(B)^\ell$ of parameters, the formula defines a function
$\llbracket \phi(\bar x\smid\bar v)\rrbracket^B:U(B)^k\to\{0,1\}$ by
\begin{equation}\label{eq:model}
\llbracket \phi(\bar x\smid\bar v)\rrbracket^B(\bar u):=
\begin{cases}
  1&\text{if } B\models\phi(\bar u\smid\bar v),\\
  0&\text{otherwise},
\end{cases}
\end{equation}
which we regard as a concept or hypothesis over our instance
space. Here $B\models\phi(\bar u\smid\bar v)$ means that $B$ satisfies
$\phi$ if the variables $\bar x$ are interpreted by the values
$\bar u$ and the variables $\bar y$ by the values $\bar v$. Depending
on the context, we often call $\llbracket \phi(\bar x\smid\bar
v)\rrbracket^B$ an $\LL$-definable model or hypothesis.

\begin{example}\label{exa:1}
  Let $B$ be an $\{E,R\}$-structure, where $E$ is a binary and $R$ a
  unary relation symbol. $B$ may be viewed as a directed graph in which some
  vertices are coloured red. Consider, for example, the graph shown in
  Figure~\ref{fig:exa1}. 

  As input for a learning algorithm, we receive a training sequence
  consisting of some vertices labelled $0$ or $1$. In our example,
  this may be the sequence $\big((a,0),(b,1),(g,0),(k,1)\big)$. From
  these training examples, we
  are supposed to figure out a global labelling function.

  Consider the first order formula
  \begin{align*}
  \phi(x\smid y_1,y_2):=&\big(R(x)\vee x=y_1\vee
  E(x,y_1)\big)\\
    &\wedge\neg\exists z(E(y_2,z)\wedge E(z,x)).
  \end{align*}
  If we take as parameters $v_1:=j,v_2:=e$, then the hypothesis $\model
  B\phi x{v_1,v_2}$ is consistent with the training examples.
\end{example}

\begin{figure}[t]
  \centering
  \begin{tikzpicture}
    [
    wvert/.style={circle,draw,minimum size=5mm,inner sep=0mm},
    rvert/.style={circle,draw,white,fill=red,minimum size=5mm,inner sep=0mm}
    ]
    \path (2,6) node[rvert] (a) {$a$} node[above=2mm] {0};
    \path (1,4.5) node[rvert] (b) {$b$} node[above left=1.5mm] {1};
    \node[rvert] (c) at (3,4.5) {$c$};
    \node[wvert] (d) at (0,3) {$d$};
    \node[rvert] (e) at (2,3) {$e$};
    \node[wvert] (f) at (4,3) {$f$};
    \path (0.5,1.5) node[wvert] (g)  {$g$} node[right=2mm] {0};
    \path (2,1.5) node[wvert] (h)  {$h$};
    \node[rvert] (i) at (4,1.5) {$i$};
    \node[wvert] (j) at (1,0) {$j$};
    \path (3,0) node[wvert] (k)  {$k$} node[right=2mm] {1};

    \draw[->,thick] (b) edge (a) edge[bend left] (e) edge (d)
    (c) edge (f) edge (a)
    (d) edge (g) 
    (e) edge[bend left] (b) edge (c) edge (g) edge (i)
    (f) edge (e) edge (i)
    (h) edge (i) edge (e) edge (j)
    (j) edge (g) edge[bend left] (d)
    (k) edge (i) edge (j);
  \end{tikzpicture}
  \caption{Structure of Example~\ref{exa:1} (red vertices appear
    grey in b/w)}
  \label{fig:exa1}
\end{figure}
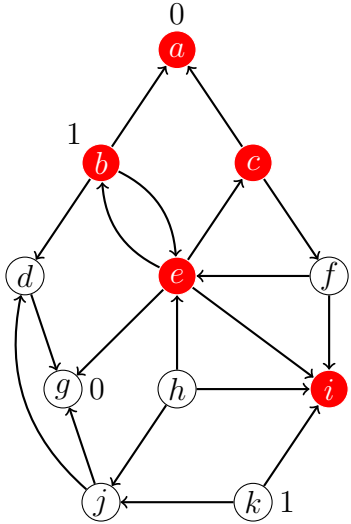

Working in such a rich declarative framework, however, it is easy to get
carried away by the expressiveness it gives. It is important,
therefore, to make sure that the models can still be learned by
efficient algorithms. There are two basic algorithmic problems, which
we may call \emph{parameter learning} (or \emph{parameter estimation})
and \emph{model learning} (or \emph{model estimation}). For both,
assume that we have a background structure $B$ and a logic $\LL$. In
the parameter learning problem, we assume a fixed $\LL$-formula
$\phi(\bar x\smid\bar y)$, and we want to find parameters that fit the
data. In the model learning problem, we are only given the data, and
we want to find an $\LL$-formula $\phi(\bar x\smid\bar y)$ as well as
parameters fitting the data. To avoid overfitting, we want to choose
the formula $\phi(\bar x\smid\bar y)$ to be as simple as possible,
according to some metric (for example, length or quantifier-rank). At
first sight, it seems that the parameter learning problem is the
simpler one, but this is not necessarily the case (as we shall discuss
in Section~\ref{sec:pl}).

These considerations suggest the following research program:
\emph{Identify logics suitable for expressing relevant models for
  machine learning and study their algorithmic learnability, that is,
  efficient learning algorithms and also complexity theoretic or
  information theoretic lower bounds.} Ideally, the logics would be
expressive enough to define all feasible models and at the same time
only permit the definition of feasible models. Note the similarity of
these desiderata with those for database query languages
(e.g.~\cite{chahar82}). The theoretical side of this program may be
viewed as a ``descriptive complexity theory of machine learning'', and
this is where our technical contributions are.

Before we describe our results, let us discuss one more technical
issue. The input of a learning algorithm consists of the training
examples, but in our framework of learning definable concepts the
algorithm also needs access to the background structure $B$. One
possible scenario is that $B$ is a fixed infinite structure, for
example the field of real numbers, and we consider an abstract
computation model where the algorithm can store an element of the
structure in a single memory cell and has access to the operations of
the structure. A second is that $B$ is a finite structure, for example
a graph describing the world wide web or a social network. In this
case, we can simply regard $B$ as part of the input, but we may think
of $B$ as still being too large to fit into main memory and only give
our algorithms limited access to $B$, such as \emph{local access} that only
allows the algorithm to retrieve the neighbours of a vertex that we already know
(we can think of this as being able to follow links). This is the
scenario we consider here.

\subsection{Our results}
We give a learning algorithm for the model learning problem for
first-order logic. The twist of our result is that if the degree of
the background structure is at most polylogarithmic, then 
the algorithm works in sublinear, in fact, polylogarithmic time, in the
size of the background structure. It came as a surprise to us that
this is possible at all. In analysing the algorithm, we take a
data-complexity point of view \cite{var82}, that is, we measure the
running time in terms of the size of the structure and hide the
dependence on the (presumably small) formula in the constants.

We only consider relational structures in this paper.
The \emph{maximum degree} $\Delta(B)$ of a structure $B$ is the maximum degree of its
Gaifman graph, in which two vertices are adjacent if they appear
together in some tuple of some relation of $B$ (see
Section~\ref{sec:log-struc} for details). The \emph{$k$-ary learning
  problem} over a background structure $B$ has instance space
$U(B)^k$. The goal is to learn an unknown target concept $C^*:U(B)^k\to\{0,1\}$.

A \emph{learning algorithm} for the $k$-ary learning problem over some
background structure $B$ receives as input a finite sequence $T$ of
training examples. In addition, we grant our learning
algorithms \emph{local access} to the background structure (in the
sense described above, see Section~\ref{sec:log-struc} for details).
We usually let $t:=|T|$ be the length of the sequence $T$.
The training examples are pairs $(\bar u,C^*(\bar u))$, where
$\bar u\in U(B)^k$. We say that $H:U(B)^k\to\{0,1\}$ is \emph{consistent}
with $T$ if for all $(\bar u,c)\in T$ we have $H(\bar u)=c$.
Given $T$, the learning algorithm is supposed to compute a \emph{hypothesis}
$H:U(B)^k\to\{0,1\}$, which in our setting is always of the form
$\model B\phi{\bar x}{\bar v}$ for some first-order formula $\phi(\bar
x\smid\bar y)$ and parameter tuple $\bar v\in U(B)^\ell$. Of course the algorithm is not supposed to return the
whole set $\model B\phi{\bar x}{\bar v}$, but just the formula
$\phi(\bar x\smid\bar y)$
and the parameter tuple $\bar v$. However, we also allow our learning
algorithms to \emph{reject} an input, to account for the situation
that after seeing the training examples the algorithm realises that its
assumption about the model was wrong, that is, there simply
is no hypothesis of the form $\model B\phi{\bar x}{\bar v}$ consistent
with the training examples.

If a learning algorithm $\mathfrak L$ returns a hypothesis
$\model B\phi{\bar x}{\bar v}$ specified by a formula
$\phi(\bar x\smid\bar y)$ and the parameter tuple $\bar v$, then this
hypothesis is useless if we cannot efficiently determine the value
$\model B\phi{\bar x}{\bar v}(\bar u)$ for a given $\bar u\in U(B)^k$. We say that the
\emph{hypotheses returned by $\mathfrak L$ can be evaluated in time
  $\mathfrak t$} if there is an algorithm that, given a pair $\phi(\bar
x\smid\bar y)$, $\bar v$ returned by $\mathfrak L$ and a tuple $\bar
u\in U(B)^k$, computes $\model B\phi{\bar x}{\bar v}(\bar u)$ in
time $\mathfrak t$.

The goal is to produce a hypothesis $H$ that \emph{generalises} well,
that is, approximates the target concept $C^*$ closely. To capture
theoretically what it means for a hypothesis to generalise well, we
will use the framework of \emph{probably  approximately
  correct} learning. However, let us first state
our main algorithmic result.

\begin{theorem}\label{theo:log}
  Let $k,\ell,q\in\NN$. Then there is a $q^*\in\NN$ and
  a learning algorithm $\mathfrak
  L$ for the $k$-ary learning problem over some finite background structure
  $B$ with the following properties.
  \begin{enumerate}
  \item If the algorithm returns a hypothesis $H$, then
    $H$ is of the form $\model B{\phi^*}{\bar x}{\bar v^*}$ for some
    first-order formula $\phi^*(\bar x\smid\bar y)$ 
    of quantifier rank at most
    $q^*$ and $\bar v^*\in U(B)^\ell$, and $H$ is
    consistent with the input
    sequence $T$ of training examples.
  \item If there is a first-order formula $\phi(\bar x\smid\bar y)$ of quantifier rank 
    $q$ and some tuple $\bar v\in U(B)^\ell$ of parameters such that 
    $\model B\phi{\bar x}{\bar v}$ is consistent with the input 
    sequence $T$, then 
    $\mathfrak L$ always returns a hypothesis and never rejects. 
  \item The algorithm runs in time $(\log n+d+t)^{O(1)}$ with only
    local access to $B$, where
    $n:=|U(B)|$ and $d:=|\Delta(B)|$ and $t:=|T|$.
  \item The hypotheses returned by $\mathfrak L$ can be evaluated in
    time $(\log n+d)^{O(1)}$ with only
    local access to $B$.
  \end{enumerate}
\end{theorem}	

Note that if the maximum degree $d$ and the length $t$ of the training
sequence are polylogarithmic in $n$, then
the overall running time of the algorithm is polylogarithmic in $n$.

The proof of Theorem~\ref{theo:log} critically relies on the locality
of first-order logic.

We also prove a generalisation of Theorem~\ref{theo:log}
(Theorem~\ref{theo:log-err}), where instead of insisting on a
consistent hypothesis, we compute, within the same polylogarithmic
time bound, a hypothesis that minimises the training error. The
\emph{training error} of a concept or hypothesis is the fraction of
training examples on which it is wrong. The algorithm we obtain
returns a hypothesis $\model B{\phi^*}{\bar x}{\bar v^*}$ for a formula
$\phi^*(\bar x\smid\bar y)$ of quantifier rank $q^*$ bounded in terms
of $q,k,\ell$ that matches the training error of the best
$\model B\phi{\bar x}{\bar v}$ for formulas $\phi$ of quantifier rank
$q$.

A variant of Theorem~\ref{theo:log}
(Theorem~\ref{theo:uni}) for bounded degree background structures
even applies to infinite structures. For this to work, we need
another level of abstraction in our computation model: we allow
it to store an element of the structure in a single memory cell and
to access it in a single computation step. We call this the \emph{uniform
  cost} measure. This is in line with the standard uniform-cost RAM
model that is also underlying the analysis of algorithmic meta
theorems for bounded degree
graphs~\cite{Seese96lineartime,gro01a,kazseg11,durschseg14,seg15}. Under
the uniform cost measure, we even obtain a learning algorithm (in fact
the same algorithm $\mathfrak L$ as in Theorem~\ref{theo:log}) running
in time $(d+t)^{O(1)}$ and producing hypotheses that can be evaluated
in time $d^{O(1)}$.

Let us briefly discuss  the implications of our results in Valiant's \cite{valiant1984theory} framework of \emph{probably approximately correct
 (PAC)} learning. A detailed technical discussion and a precise
statement of our results follows in
Section~\ref{sec:pac}. The basic assumption of PAC-learning is that there is an unknown probability distribution on the instance space
 and that instances are drawn independently from this distribution;
 the training instances as well as the new instances that we want to
 classify with our hypothesis. The \emph{generalisation error} of a
 hypothesis is then defined as the probability that the hypothesis is
 wrong on an instance drawn randomly from the distribution. A \emph{PAC-learning algorithm}
 is supposed to generate, with high \emph{confidence} $1-\delta$, a hypothesis with a small generalisation
 error $\epsilon$. The confidence is the probability taken over the randomly
 chosen training examples that the algorithm succeeds. The number $t$ of
 training examples the algorithm has access to is bounded in terms of
 the error parameter $\epsilon$ and the confidence parameter $\delta$,
 and the running time is supposed to be polynomial in the size of its
 input, which in our setting is $t\log n$, or just $t$ under the uniform
 cost measure. To obtain a PAC-learning algorithm, we usually need to
 make assumptions about the target concept. We say that a class $\CC$
 of concepts is \emph{PAC-learnable} if there is a learning algorithm
 that meets the PAC-criterion whenever the target concept is from
 $\CC$, regardless of the probability distribution.

 \begin{corollary}\label{cor:pac}
   Let $d,k,\ell,q\in\NN$.
   For a background structure $B$ of maximum degree at most $d$, let
   $\CC$ be the class of all concepts $\llbracket\phi(\bar
   x\smid\bar v)\rrbracket^B$, where $\phi(\bar x\smid\bar y)$ is a
   first-order formula 
    of quantifier rank at most
    $q$ with $|\bar x|=k$, $|\bar y|=\ell$ and $\bar v\in U(B)^\ell$.

    Then $\CC$ is PAC-learnable by
    an algorithm that only has local access to $B$ and runs in time
    polynomial in $1/\epsilon$ and $\log1/\delta$.
 \end{corollary}

Theorem~\ref{theo:uni} is a more detailed statement of this result. In addition to Theorem~\ref{theo:log}
(or Theorem~\ref{theo:uni} for the uniform cost measure), the corollary also relies on a result
from \cite{grohe2004learnability} stating that first-order definable
set families on graphs of bounded degree, such as the family $\CC$
in the corollary, have bounded VC-dimension, which by a general
theorem due to  Blumer, Ehrenfeucht, Haussler and
Warmuth~\cite{bluehrhau+89} implies that a number of training examples
only depending on $\epsilon$ and $\delta$ is sufficient.

We also prove a PAC-learning result for background structures with
polylogarithmic degree (Theorem~\ref{theo:pac-log}) and a
generalisation in the so-called ``agnostic'' PAC-learning framework
which deals with target concepts that are random variables (Theorem~\ref{theo:pac-log-err}).

\subsection{Related Work}
Closest to our framework is that of inductive logic programming (ILP)
(see, for example,
\cite{cohpag95,kiedze94,mug91,mug92,mugder94}). However, there are
important differences. First of all, our framework is by no means
restricted to first-order logic, and in future work we intend to look
at other languages which may be more suitable for expressing concepts
relevant in the machine learning context. However, the present paper
is only concerned with first-order logic. A second difference that is
more significant for this paper is that we represent the background
knowledge in a background structure, whereas in ILP it is represented
by a background theory. This leads to quite different
intuitions. Whereas the scarce positive PAC-learnability results in the
ILP framework are mostly obtained by syntactically restricting the formulas
defining models  (see, however, \cite{hortur01,horslotur03}), our results exploit structural restrictions---small
degree---of the background structure.

To the best of our knowledge, the idea of a learning algorithm having
only local access to the background structure is new here, and this is
precisely what enables the polylogarithmic running time of our
algorithms. We are not aware of results from the ILP context that lead
to algorithms which are sublinear in the background knowledge. The
local access approach seems related to ideas used in property testing
on bounded degree graphs (see, for example,
\cite{golron02,golron11,czushasoh09}). We leave it for future work to
explore possible connections.

Our framework of learning definable concepts over a background
structure has been considered before by Grohe and Tur\'an~\cite{grohe2004learnability}. However,
the results from \cite{grohe2004learnability} are not algorithmic. They only bound the
VC-dimension of definable concept classes over certain classes of
structures. As mentioned above, we use one of the results of \cite{grohe2004learnability}
in the proof of Corollary~\ref{cor:pac}.

An alternative logical learning framework, also with a strong
foundation in descriptive complexity theory, has recently been
proposed by Crouch, Immerman, and Moss~\cite{croimmmos10} and Jordan
and Kaiser~\cite{jorkai16} (also see
\cite{kai12,jorkai13}). Here the goal is to learn a logical
reduction between structures; instances are pairs of structures.
There is is other interesting recent work on learning in a logic setting in database
theory (for example, \cite{aboangpap+13,bonciusta16}) and verification
(for example, \cite{lodmadnei16,garneimadrot16}).
While there is no direct technical connection between our work and
these research directions, they all seem to be similar in
spirit. Exploring the exact technical relations remains future work.

\section{Background from Logic}\label{sec:log}

\subsection{Structures}
\label{sec:log-struc}

In this paper, we only consider relational structures. A
\emph{relational vocabulary} is a finite set $\rho$ of relation
symbols, each with a prescribed \emph{arity}. A
\emph{$\rho$-structure} $A$ consists of a set $U(A)$, the
\emph{universe} of $A$, and for each $k$-ary $R\in\rho$ a $k$-ary
relation $R(A)\subseteq U(A)^k$. A structure $A$ is \emph{finite} if
its universe $U(A)$ is finite, and the \emph{order} of $A$ is
$|A|:=|U(A)|$. (For infinite $A$, we let $|A|:=\infty$.)

The \emph{union} of two $\rho$-structures $A,B$ is the
$\rho$-structure $A\cup B$ with universe $U(A\cup B):=U(A)\cup U(B)$
and relations $R(A\cup B):=R(A)\cup R(B)$ for all $R\in\rho$. The
intersection $A\cap B$ is defined similarly.  A \emph{substructure} of
a $\rho$-structure $A$ is a $\rho$-structure $B$ with
$U(B)\subseteq U(A)$ and $R(B)\subseteq R(A)$ for all $R\in\rho$. For
a subset $V\subseteq U(A)$, the substructure \emph{induced} by $A$ on
$V$ is the the structure $A[V]$ with universe $U(A[V])=V$ and
$R(A[V]):=R(A)\cap V^k$ for each $k$-ary $R\in\rho$.

The \emph{Gaifman graph} of a $\rho$-structure $A$ is the graph $G_A$
with vertex set $V(G_A):=U(A)$ and an edge $uv$ for all $u,v\in U(A)$
such that $u\neq v$ and there is a $k$-ary relation symbol $R\in\rho$
and a tuple $(v_1,\ldots,v_k)\in R(A)$ with
$u,v\in\{v_1,\ldots,v_k\}$. The Gaifman graph allows us to transfer
graph theoretic notions from graphs to arbitrary relational
structures. In particular, the \emph{degree} $\deg^A(u)$ of an element
$u\in U(A)$ is the number of neighbours of $u$ in $G_A$, and the
\emph{maximum degree} $\Delta(A)$ is $\max\{\deg^A(u)\mid u\in U(A)\}$
if this maximum exists and $\infty$ if it does not. The
\emph{distance} $\dist^A(u,v)$ between two elements $u,v\in U(A)$ in
$A$ is the length of the shortest path from $u$ to $v$ in $G_A$, and
$\infty$ if there is no path from $u$ to $v$. If $\dist^A(u,v)=1$ then
we say that $u$ is a \emph{neighbour} of $v$.

The \emph{$r$-neighbourhood} $N_r^A(u)$ of $u$ in $A$ is the set of
all vertices of distance at most $r$ from $u$. For a tuple
$\bar u=(u_1,\ldots,u_k)$, we let
$N_r^A(\bar u):=\bigcup_{i=1}^kN^A_r(u_i)$. To avoid cluttering the
notation even more, we also use $N_r^A(\bar u)$ to denote the induced
substructure $A[N_r^A(\bar u)]$.

In all these notations we omit the superscript ${}^A$ if the structure
$A$ is clear from the context.

Let us briefly review the computation model described in the
introduction. We say that an algorithm $\mathfrak A$ has \emph{local
  access} to a $\rho$-structure $A$ it can query an oracle in the
following two ways.
\begin{description}
\item[\mdseries\slshape  Relation queries:]Is $(u_1,\ldots,u_k)\in R$?
\item[\mdseries\slshape Neighbourhood queries:] Return a list of all
  neighbours for a given $u\in U(A)$.
\end{description}
A relation query requires constant time. A neighbourhood query requires
time proportional to the size of the (representation of the) answer,
which is the degree of $u$ times the space required to store a single
element of $A$. With the \emph{uniform cost measure}, storing an element
of $A$ requires constant space, and with the \emph{logarithmic cost
  measure} it requires space $O(\log|A|)$. Unless explicitly stated
otherwise, we assume the logarithmic cost measure.

\subsection{First-Order Logic}
Let us briefly review the definition of \emph{first-order logic}
$\FO$. \emph{First-order formulas} of vocabulary~$\rho$ are formed
from atomic formulas~$x=y$ and $R(x_1,\ldots,x_k)$, where~$R\in\rho$
is a~$k$-ary relation symbol and~$x,y,x_1,\ldots,x_k$ are variables by
the Boolean connectives~$\neg$~(negation),~$\wedge$ (conjunction),
$\vee$ (disjunction), $\to$ (implication) and existential and universal
quantification~$\exists x,\forall x$, respectively, all with the usual
semantics. The set of all first-order formulas of vocabulary $\rho$ is
denoted by~$\FO[\rho]$. The free variables of a formula are those not
in the scope of a quantifier, and we write~$\phi(x_1,\ldots,x_k)$ to
indicate that the free variables of the formula~$\phi$ are among
$x_1,\ldots,x_k$. A \emph{sentence} is a formula without free
variables. We write $A\models\phi(u_1,\ldots,u_k)$ to denote that $A$
satisfies $\phi$ if $x_i$ is interpreted by $u_i$. As explained in the
introduction, when describing (machine learning) models, we partition
the free variables of a formula into \emph{instance variables} and \emph{parameter
variables}, and we use a semicolon to separate the two parts, as in
$\phi(x_1,\ldots,x_k\smid y_1,\ldots,y_k)$ or
$\phi(\bar x\smid\bar y)$. We use the notation $\llbracket\phi(\bar
x\smid\bar v)\rrbracket^B$ introduced in \eqref{eq:model}
for the instantiation of the model with parameters $\bar v$ in a
background structure $B$.

The \emph{quantifier rank} of a first-order formula
$\phi$ is the nesting depth of quantifiers in $\phi$.

\subsection{Locality}
Let us fix a vocabulary $\rho$.  An $\FO[\rho]$-formula~$\psi(\bar x)$
is \emph{$r$-local} if for all $\rho$-structures $A$ and all tuples
$\bar u$ of elements,
\[
A\models\psi(\bar
u)\iff N_r(\bar u)\models\psi(\bar u).
\]  
A formula is \emph{local} if it is $r$-local for some
$r$. %
For all~$r\ge 0$ there is an $\FO[\rho]$-formula~$\delta_{\le r}(x,y)$
of quantifier rank $O(\log r)$ stating that the distance between~$x$
and~$y$ is at most~$r$. We write~$\delta_{>r}(x,y)$ instead
of~$\neg\delta_{\le r}(x,y)$. A \emph{basic local sentence} of
\emph{radius} $r$ is a first-order sentence of the form
\begin{equation}
  \label{eq:basic-local}
  \exists x_1\ldots\exists
x_k\big(\bigwedge_{1\le i<j\le
  k}\delta_{>2r}(x_i,x_j)\wedge\bigwedge_{i=1}^k\psi(x_i)\big),
\end{equation}
where $\psi$ is~$r$-local.  

\begin{theorem}[Gaifman's Locality Theorem~\cite{gaifman1982local}]
  Every first-order formula  is equivalent to a Boolean combination of
  basic local sentences and local formulas. 
\end{theorem}

This notion of locality defined above is semantical, but we also need a syntactical
version. 
The \emph{radius-$r$ relativisation} of an $\FO[\rho]$-formula
$\phi(x_1,\ldots,x_k)$ is the formula $\phi_{[\le r]}(x_1,\ldots,x_k)$ obtained from $\phi$ by
replacing each subformula $\exists y\psi$ by the formula $
\exists y\left(\bigvee_{i=1}^k\delta_{\le r}(x_i,y)\wedge\psi\right)
$
and every subformula $\forall y\psi$ by $
\forall y\left(\bigvee_{i=1}^k\delta_{\le r}(x_i,y)\to\psi\right).
$
Note that for every $\phi(\bar x)$ the radius-$r$ relativisation
$\phi_{[\le r]}(\bar x)$ is $r$-local. Moreover, if
$\phi(\bar x)$ is $r$-local then $\phi(\bar x)$ and $\phi_{[\le
  r]}(\bar x)$ are equivalent. Note that the transition from
$\phi(\bar x)$ to $\phi_{[\le
  r]}(\bar x)$ increases the quantifier rank by a factor of $O(\log r)$.

An $\FO[\rho]$-formula $\phi(\bar x)$ is \emph{syntactically
  $r$-local} if it is the radius-$r$ relativisation of some
$\FO[\rho]$-formula $\phi'(\bar x)$.
Every syntactically $r$-local formula is $r$-local, and
conversely, every $r$-local formula of quantifier rank
$q$ is equivalent to a syntactically $r$-local formula of quantifier
rank $O(q\cdot\log r)$ (its own radius-$r$ relativisation).

We say that a basic local sentence of the form \eqref{eq:basic-local}
is \emph{syntactically basic local} if the $r$-local formula $\psi(x)$
is syntactically $r$-local.
A formula is in \emph{Gaifman normal form} if it is a  Boolean combination of
  syntactically basic local sentences and syntactically local
  formulas. The \emph{locality radius} of a formula $\phi$ in Gaifman normal
  form is the least $r$ such that all basic local sentences in
  $\phi$ have radius at most $r$ and all  local formulas are
  syntactically $r'$-local form some $r'\le r$.

\subsection{Types}
Let $A$ be a $\rho$-structure and $\bar u=(u_1,\ldots,u_k)\in U(A)^k$.
For every $q\ge 0$, the \emph{(first-order) $q$-type of
  $\bar u$ in $A$} is the set $\tp_q(A,\bar u)$ of all
$\phi(x_1,\ldots,x_k)\in\FO[\rho]$ of quantifier rank at most $q$ such
that $A\models\phi(u_1,\ldots,u_k)$. Types are infinite sets of
formulas, but we can syntactically \emph{normalise} formulas in such a
way that there are only finitely many normalised formulas of fixed
quantifier rank and with a fixed set of free variables, and that every
formula can effectively be transformed into an equivalent normalised
formula of the same quantifier rank. We represent a type by the set of
normalised formulas it contains.

We need the following Feferman-Vaught style composition lemma
\cite{fefvau59} (also see \cite{mak04}). For tuples $\bar
u=(u_1,\ldots,u_k)$ and $\bar v=(v_1,\ldots,v_\ell)$, we let $\bar
u\bar v:=(u_1,\ldots,u_k,v_1,\ldots,v_\ell)$.

\begin{lemma}[Composition Lemma~\cite{fefvau59}]\label{lem:comp}
  Let $A,A',B,B'$ be $\rho$-structures such that $A\cap B=\emptyset$
  and $A'\cap B'=\emptyset$. Let $\bar u\in U(A)^k$, $\bar v\in
  U(B)^\ell$, $\bar u'\in U(A')^k$, and $\bar v'\in U(B')^\ell$ such
  that $\tp_q(A,\bar u)=\tp_q(A',\bar u')$ and $\tp_q(B,\bar
  v)=\tp_q(B',\bar v')$. Then 
  \[
  \tp_q(A\cup B,\bar u\bar v)=\tp_q(A'\cup B',\bar u'\bar v').
  \]
\end{lemma}

We also need a ``local'' version of types. Let $A$ be a
$\rho$-structure and $\bar u\in U(A)^k$, and let $q,r\ge 0$. The
\emph{local $(q,r)$-type of
  $\bar u$ in $A$} is the set $\ltp_{q,r}(G,\bar u)$ of all
syntactically $r$-local
$\phi(\bar x)\in\FO[\rho]$ of quantifier rank at most $q$ such
that $A\models\phi(\bar u)$, or equivalently, $N_r^A(\bar u)\models \phi(\bar u)$. 

As a corollary to Lemma~\ref{lem:comp}, we obtain the following
composition lemma for local types.

\begin{corollary}[Local Composition Lemma]\label{cor:comp}
  Let $A,A'$ be $\rho$-structures, $\bar u\in U(A)^k$, $\bar v\in
  U(A)^\ell$, $\bar u'\in U(A')^k$, and $\bar v'\in U(A')^\ell$ such
  that $N_r(\bar u)\cap N_r(\bar v)=\emptyset$ and $N_r(\bar u')\cap N_r(\bar v')=\emptyset$
  and 
$\ltp_{q,r}(A,\bar u)=\ltp_{q,r}(A',\bar u')$ and $\ltp_{q,r}(A,\bar v)=\ltp_{q,r}(A',\bar v')$.
  Then 
  \[
  \ltp_{q,r}(A,\bar u\bar v)=\ltp_{q,r}(A',\bar u'\bar v').
  \]
\end{corollary}

\section{Parameter Learning}
\label{sec:pl}

Recall the two different modes of learning we described in the
introduction: \emph{parameter learning} and \emph{model learning}. In
our context, for the parameter learning problem we assume that we have
 a fixed $\FO[\rho]$-formula $\phi(\bar x\smid\bar y)$. The input to
 our learning algorithm is a sequence $T$ of training examples over
 some background structure $B$ to which we have local access. Our goal
 is to find a tuple $\bar v$ such that $\model{B}{\phi}{\bar x}{\bar
   v}$ is consistent with $T$ (or at least approximately consistent).

The following simple example shows that parameter learning requires
reading the whole background structure and thus is not possible in
sublinear time, whereas our model learning algorithms only need
polylogarithmic time.

\begin{example}\label{exa:pl}
  Let $P$ be a unary relation symbol, and let $\phi(x\smid y):=P(y)$.
  Then for every $\{P\}$-structure $B$ and every $v\in U(B)$, the
  function $\model B\phi xv$ is constant $1$ if $v\in P(B)$ and
  constant $0$ otherwise.

  Now suppose that the background structure $B$ is such that
  $P(B)=\{v^*\}$ for some $v^*\in U(B)$, and the unknown target
  concept is $\model B\phi x{v^*}$, that is, constant $1$. Then our learning algorithm
  receives only positive training examples, and it needs to find the
  parameter $v^*$. However, unless $v^*$ happens to be one of the
  training examples, this requires reading the whole structure $B$ in
  the worst case. If the algorithm only has local access to $B$, it
  is actually impossible to find $v^*$, because the graph $G_B$ is the
  trivial graph with vertex set $U(B)$ and empty edge set.
\end{example}

Thus parameter learning is not possible in our setting with only local
access to the background structure. However, there is an intermediate
mode of learning between parameter and model learning, where we assume
that we know the formula $\phi(\bar x\smid\bar y)$ defining the target
concept, but are still allowed to modify it when formulating our
hypothesis. This is potentially much easier than constructing the
formula from scratch, as we are required to do in the model learning
setting.

For instance, in Example~\ref{exa:pl}, looking at $\phi$ we
immediately know that the target concept is constant, and just from
one labelled example we know if it is $0$ or $1$. Then
we can either return the universally true formula
$\phi'(x\smid):=(x=x)$ or the universally false formula
$\phi''(x\smid):=(\neg x=x)$ as our hypothesis (we do not even need a
parameter here).

\section{Model Learning}
In this section, we look at the model learning problem for first-order
logic over structures of small degree. We prove Theorem~\ref{theo:log}
and several variants and generalisations of it.

\subsection{Consistent Hypotheses}
We start by proving Theorem~\ref{theo:log}.

Throughout this section, we fix $k,\ell,q\in\NN$ and a vocabulary
$\rho$. Let $q^*,r^*\in\NN$ such that every $\FO[\rho]$-formula $\phi$
with at most $k+\ell$ free variables and quantifier rank at most $q$
is equivalent to a formula $\phi^*$ in Gaifman normal form of locality radius at
most $r^*$ such that the quantifier rank of every syntactically local formula in the
Boolean combination $\phi^*$ has quantifier rank at most $q^*$. Such
$q^*,r^*$ exist by Gaifman's theorem and because up to logical
equivalence there only exist finitely many $\FO[\rho]$-formulas of
quantifier rank at most $q$ with at most $k+\ell$ free variables.

\begin{lemma}\label{lem:local}
  Let $B$ be a $\rho$-structure, and let $\psi(\bar z)$ be an 
  $\FO [\rho]$-formula
  of quantifier rank at most $q$ and with $m:=|\bar z|\le
  k+\ell$. Then for all $\bar w,\bar w'\in U(A)^m$ 
  \begin{align*}
  &\ltp_{q^*,r^*}(B,\bar w)=\ltp_{q^*,r^*}(B,\bar w')\implies\\
    &\hspace{1.8cm}
  \big(B\models\psi(\bar w)\iff B\models\psi(\bar w')\big)
  \end{align*}
\end{lemma}

\begin{proof}
  It follows from Gaifman's Locality Theorem and  the choice of
  $q^*,r^*$ that $\big(B\models\psi(\bar w)\iff B\models\psi(\bar
  w')\big)$ if $\big(B\models\chi(\bar w)\iff B\models\chi(\bar
  w')\big)$ for all syntactically $r^*$-local formulas $\chi(\bar z)$ of quantifier
  rank at most $q^*$. If $\ltp_{q^*,r^*}(B,\bar w)=\ltp_{q^*,r^*}(B,\bar
  w')$ then the latter equivalence holds.
\end{proof}

To simplify the
presentation, we fix a background structure $B$. We also fix the
length $t\ge1$ of our training sequences. Of course our algorithm
will neither depend
on the specific structure $B$ nor on $t$.

We fix a tuple $\bar x=(x_1,\ldots,x_k)$ of instance variables and a
tuple $\bar y=(y_1,\ldots,y_\ell)$ of parameter variables.
We let $\Phi$ be the set of all $\FO[\rho]$-formulas $\phi(\bar
x\smid\bar y)$ of quantifier rank at most $q$ and
\[
\CC:=\big\{\model B\phi{\bar x}{\bar v}\bigmid \phi(\bar x\smid\bar
y)\in\Phi,\bar v\in U(B)^\ell\big\},
\]
Similarly, we let $\Phi^*$ be the set of all syntactically $r$-local $\FO[\rho]$-formulas $\phi^*(\bar
x\smid\bar y)$ of quantifier rank at most $q^*$ and 
\[
\CC^*:=\big\{\model B{\phi^*}{\bar x}{\bar v}\bigmid \phi(\bar x\smid\bar
y)\in\Phi^*,\bar v\in U(B)^\ell\big\}.
\]
Moreover, we let $\CT:=\big(U(B)^k\times\{0,1\}\big)^t$ be
the set of all training sequences
of length $t$ for the $k$-ary learning problem over $B$.
For every $T=\big((\bar u_1,c_1),\ldots,(\bar u_t,c_t)\big) \in\CT$
and $r\in\NN$ we let
\[
N_{r}(T):=\bigcup_{i=1}^tN_{r}(\bar u_i).
\]
Recall that
$T\in\CT$ is \emph{consistent} with $C\subseteq U(B)^k$ if for all
$(\bar u,c)\in T$ we have $C(\bar u)=c$.

What we need to prove is that for all $T\in\CT$, if there is a
$C\in\CC$ that is consistent with $T$  then we can find a
$C^*\in\CC^*$ consistent with $T$ within the time bounds specified in
Theorem~\ref{theo:log}.

The following lemma is the crucial step in our proof. 

\begin{lemma}\label{lem:main}
  Let $T\in\CT$ be consistent with some $C\in\CC$. 
  Then there is a formula $\phi^*(\bar x\smid\bar y)\in\Phi^*$ and a tuple
  $\bar v^*\in N_{2\ell r^*}(T)^\ell$ such that $\model B{\phi^*}{\bar
    x}{\bar v^*}$ is consistent with $T$.
\end{lemma}

\begin{proof}
  Let $T=\big((\bar u_1,c_1),\ldots,(\bar u_t,c_t)\big)$.  Let
  $\phi(\bar x\smid\bar y)\in\Phi$ and
  $\bar v=(v_1,\ldots,v_\ell)\in U(B)^\ell$ such that
  $C=\model B\phi{\bar x}{\bar v}$ is consistent with $T$.

  For some $m\le \ell$, we
  define $v^{(1)},\ldots,v^{(m)}\in\{v_1,\ldots,v_\ell\}$ and $N^{(0)},N^{(1)},\ldots,N^{(m)}\subseteq U(B)$ as follows: we
  let $N^{(0)}:=N_{r^*}(T)$. Now suppose that $N^{(i)}$ is already
  defined. If there is a
  $v\in\{v_1,\ldots,v_\ell\}\setminus\{v^{(1)},\ldots,v^{(i)}\}$ such
  that $\dist^B(v,N^{(i)})\le r^*$, then we pick such a $v$ (arbitrarily
  if there are more than one) and let $v^{(i+1)}:=v$ and
  $N^{(i+1)}:=N^{(i)}\cup N_{r^*}(v^{(i+1)})$. If there is no such $v$, we
  let $m:=i$ and stop the construction. 

  We let $N^\circ:=N^{(m)}$. To simplify the notation,
  we further assume (without loss of generality) that $v^{(i)}=v_i$
  for all $i\in[m]$. We let $\bar v^\circ:=(v_1,\ldots,v_m)$ and $\bar
  v^\bullet:=(v_{m+1},\ldots,v_\ell)$. Possibly, $\bar v^\circ$ or
  $\bar v^\bullet$ is the empty tuple.  Observe
  that $\bar v^\circ\in N_{2\ell r^*}(T)^m$ and 
  \[
  N^\circ\subseteq N_{2\ell r^*+r^*}(T)
  \]
  and 
  \begin{equation}
    \label{eq:1}
   N^\circ=\bigcup_{i=1}^tN_{r^*}(\bar u_i)\cup\bigcup_{i=1}^mN_{r^*}(v_i).
  \end{equation}
  Furthermore, 
  \begin{equation}
    \label{eq:2}
    N_{r^*}(\bar v^\bullet)\cap N^\circ=\emptyset.
  \end{equation}

  \begin{claim}
    Let $i,j\in[t]$ such that
    \[
    \ltp_{q^*,r^*}(B,\bar u_i\bar
    v^\circ)=\ltp_{q^*,r^*}(B,\bar u_j\bar v^\circ).
    \]
    Then $c_i=c_j$.

    \proof
    It follows from \eqref{eq:1} and \eqref{eq:2} that $N_{r^*}(\bar
    u_i\bar v)\cap N_{r^*}(\bar v^\bullet)=\emptyset$ and $N_{r^*}(\bar
    u_j\bar v)\cap N_{r^*}(\bar v^\bullet)=\emptyset$. Thus by the
    Local Composition Lemma (Corollary~\ref{cor:comp}), we have
     \[
    \ltp_{q^*,r^*}(B,\bar u_i\bar v)=\ltp_{q^*,r^*}(B,\bar u_j\bar v).
    \]
    Thus by Lemma~\ref{lem:local}, we have $B\models\phi(\bar
    u_i\smid\bar v)\iff B\models\phi(\bar
    u_j\smid\bar v)$. This implies the claim.
    \uend
  \end{claim}

  Now let $P\subseteq [t]$ be the set of indices of the positive
  examples, that is, $P:=\{p\in[t]\mid c_p=1\}$. For every $p\in P$,
  let $\theta_p(\bar x,\bar y^\circ)$, where $\bar y^\circ:=(y_1,\ldots,y_m)$, be the conjunction of
  all normalised formulas in the type
  $\ltp_{q^*,r^*}(B,\bar u_i\bar
  v^\circ)$.
  Then for all $\bar u'\in U(B)^k,\bar v'\in U(B)^m$ we have
  \begin{equation}
    \label{eq:3}
      B\models\theta_p(\bar u',\bar v')\iff 
  \ltp_{q^*,r^*}(B,\bar u'\bar
  v')=\ltp_{q^*,r^*}(B,\bar u_p\bar v^\circ). 
  \end{equation}
  Now we let $\phi^\circ(\bar x\smid\bar y^\circ):=\bigvee_{p\in
    P}\theta_p(\bar x,\bar y^\circ)$. Then it follows from Claim~1 and
  \eqref{eq:3} that $\model B{\phi^\circ}{\bar x}{\bar v^\circ}$ is
  consistent with $T$. Furthermore, all the $\theta_p$ and hence
  $\phi^\circ$ are syntactically $r^*$-local of quantifier rank at
  most $q^*$.

  It remains to transform $\phi^\circ(\bar x\smid\bar y^\circ)$ into a
  formula $\phi^*(\bar x\smid\bar y)$ with the right number of
  parameter variables. We simply do this by adding redundant variables,
  but we have to be careful that the resulting formula is still
  syntactically $r^*$-local. Since $\phi^\circ(\bar x\smid\bar y^\circ)$ is
  syntactically $r^*$-local, all its quantifiers are relativised to the
  $r^*$-neighbourhood of the free variables, that is, of the form
  \begin{equation}
    \label{eq:4}
     \exists z\left(\left(\bigvee_{i=1}^k\delta_{\le
        r^*}(x_i,z)\vee \bigvee_{j=1}^m\delta_{\le
        r^*}(y_j,z)\right)\wedge\ldots\right)
  \end{equation}
   or
   \begin{equation}
     \label{eq:5}
  \forall z\left(\left(\bigvee_{i=1}^k\delta_{\le
        r^*}(x_i,z)\vee \bigvee_{j=1}^m\delta_{\le
        r^*}(y_j,z)\right)\to\ldots\right).  
   \end{equation}
  To obtain $\phi^*(\bar x\smid\bar y)$ from $\phi^\circ(\bar
  x\smid\bar y^\circ)$, we replace \eqref{eq:4} by 
  \begin{align*}
  &\exists z\left(\left(\bigvee_{i=1}^k\delta_{\le
      r^*}(x_i,z)\vee \bigvee_{j=1}^\ell\delta_{\le
      r^*}(y_j,z)\right)\right.\\
    &\hspace{1cm}\wedge\left. \left(\bigvee_{i=1}^k\delta_{\le
      r^*}(x_i,z)\vee \bigvee_{j=1}^m\delta_{\le
      r^*}(y_j,z)\right)\wedge\ldots\right),
  \end{align*}
  and similarly for the universal quantifier in \eqref{eq:5}. Then
  $\phi^*(\bar x\smid\bar y)$ is syntactically $r^*$-local and has the
  same quantifier rank as $\phi^\circ(\bar x\smid\bar y^\circ)$. Hence
  $\phi^*(\bar x\smid\bar y)\in\Phi^*$. Moreover, for all $\bar u'\in
  U(B)^k$ and all $\bar v'=(v'_1,\ldots,v'_\ell)\in U(B)^\ell$ we have 
  \[
  B\models \phi^*(\bar u'\smid\bar v')\iff B\models\phi^{\circ}(\bar
  u\smid (v_1',\ldots,v'_m)).
  \]
  We choose an arbitrary $v\in N_{2\ell r^*}(T)$ and let $\bar
  v^*=(v_1,\ldots,v_m,\overbrace{v,\ldots,v}^{(\ell-m)\text{
      times}})$. Then for all $i\in[t]$ we have 
 \[
  B\models \phi^*(\bar u_i\smid\bar v^*)\iff B\models\phi^{\circ}(\bar
  u_i\smid \bar v^\circ)\iff c_i=1.
  \]
  Hence $\model B{\phi^*}{\bar x}{\bar v^*}$ is consistent with $T$.
\end{proof}

\begin{figure}
  \textbf{Algorithm} $\mathfrak L$
  \begin{algorithmic}[1]
    \Require 
    Training sequence
    $T\in\CT$,\newline
    local access to background structure $B$
    \State $N\gets N_{2\ell r^*}(T)$
    \ForAll{$\bar v^*\in N^\ell$}
    \ForAll{$\phi^*(\bar x\smid\bar y)\in\Phi^*$}
    \State $consistent\gets\textbf{true}$
    \ForAll{$(\bar u,c)\in T$}
    \If{$\hspace{-1em}\begin{array}[t]{c@{\,}l@{\,}}&(N_{r^*}(\bar u\bar v^*)\models\phi^*(\bar u\smid\bar v^*)\text{ and
                            }c=0)\\
           \text{or}&(N_{r^*}(\bar u\bar v^*) \not\models\phi^*(\bar
                      u\smid\bar v^*)\text{ and }c=1)\end{array}$}
                  \smallskip
    \State $consistent\gets\textbf{false}$
    \EndIf
    \EndFor
    \If{$consistent$}
    \Return $\phi,\bar v$
    \EndIf
    \EndFor
    \EndFor
    \State\textbf{reject}
  \end{algorithmic}
  \caption{Learning algorithm $\mathfrak L$ of Theorem~\ref{theo:log}}
  \label{fig:L}
\end{figure}

\begin{proof}[Proof of Theorem~\ref{theo:log}]
  The pseudocode for our
  learning algorithm $\mathfrak L$ is shown in Figure~\ref{fig:L}. The
  algorithm proceeds by brute-force: it goes through all formulas
  $\phi^*\in\Phi^*$ and all tuples $\bar v^*\in N_{2\ell r^*}(T)$ and checks, in lines 4--7, if $\model
  B{\phi^*}{\bar x}{\bar v^*}$ is consistent with $T$. If it is, the
  algorithm returns $\phi^*,\bar v^*$, otherwise it proceeds to the next
  $\phi^*,\bar v^*$. If it does not find any consistent $\phi^*,\bar
  v^*$, it rejects. To see that the consistency test in lines 4--7 is correct,
  note that \[N_{r^*}(\bar u\bar v^*) \models\phi^*(\bar u,\bar
  v^*)\iff B\models\phi^*(\bar u,\bar v^*),\] because $\phi^*$ is $r^*$-local.
  
  Hence the algorithm is correct, that is, satisfies conditions (1)
  and (2) of Theorem~\ref{theo:log}: it obviously satisfies (1), and
  it follows from Lemma~\ref{lem:main} that it satisfies (2).

  Note that the set $N$ (in line 1) can be computed from $T$ with
  only local access to $B$.

  To analyse the running time of $\mathfrak L$, let $n:=|B|$ and
  $d:=|\Delta(B)|$ and $t:=|T|$. Note that for all $\bar u\in
  U(B)^k,\bar v^*\in U(B)^\ell$ we have
  \[
  |N_{r^*}(\bar u\bar v^*) |\le (k+\ell)\cdot 2d^{r^*}.
  \]
  Thus the representation size of the substructure  $N_{r^*}(\bar
  u\bar v^*)$ of $B$ is $O((k+\ell)\cdot d^{r^*}\cdot \log n)$, which
  is $(\log n+d)^{O(1)}$ if we treat $k,\ell,r^*$ as constants. It
  requires time polynomial in the size of $N_{r^*}(\bar
  u\bar v^*)$ to test if the structure satisfies $\phi^*$. Thus
  the overall running time of lines 4--7 is $(\log n+d)^{O(1)}\cdot
  t$. We have
  \[
  |N|\le2tkd^{2\ell r^*}=(t+d)^{O(1)}
  \]
  and $|\Phi^*|=O(1)$.  Hence the two outer loops add a factor of
  $(t+d)^{O(1)}$, and the overall runtime is 
  \[
  (t+d)^{O(1)}\cdot (\log n+d)^{O(1)}\cdot
  t=(\log n+ d+t)^{O(1)}.
  \]
  This proves Theorem~\ref{theo:log}(3).

  Finally, any hypothesis $\model B{\phi^*}{\bar x}{\bar v^*}$ returned by
  $\mathfrak L$ can be evaluated in time $(\log n+d)^{O(1)}$ with
  only local access to $B$, because the formula $\phi^*$ is
  $r^*$-local, and thus to compute $\model B{\phi^*}{\bar
    x}{\bar v^*}(\bar u)$ we only need to look at the substructure
  $N_{r^*}(\bar u\bar v^*)$ of $B$. This proves Theorem~\ref{theo:log}(4).
\end{proof}

Let us now analyse the algorithm under  the uniform cost
measure. Everything remains unchanged, except that the log-factors in
the running time disappear. One advantage of the uniform cost model is
that we can even apply it to infinite background structures $B$. We
obtain the following theorem.

\begin{theorem}\label{theo:uni}
  Let $k,\ell,q\in\NN$. Then there is a $q^*\in\NN$ and
  a learning algorithm $\mathfrak
  L$ for the $k$-ary learning problem over some (possibly infinite) background structure
  $B$ with properties 1) and 2) of Theorem~\ref{theo:log} and the
  following two properties.
  \begin{enumerate}
  \item[3u)]The algorithm runs in time $(d+t)^{O(1)}$ under the uniform
    cost measure with only
    local access to $B$, where
    $d:=|\Delta(B)|$ and $t:=|T|$.
  \item[4u)] The hypotheses returned by $\mathfrak L$ can be evaluated in
    time $d^{O(1)}$ under the uniform cost measure with only
    local access to $B$.
  \end{enumerate}
\end{theorem}

\subsection{Minimising the Training Error}
We continue to work in the same framework as before, that is, we
consider the $k$-ary learning problem over a background structure
$B$. Let $T=\big((\bar u_1,c_1),\ldots,(\bar u_t,c_t)\big) \in
\big(U(B)^k\times\{0,1\}\big)^t$ be a training sequence. The
\emph{training error} of a hypothesis $H:U(B)^k\to\{0,1\}$ on $T$ is the
fraction of examples on which $H$ is wrong, that is,
\[
\err_T(H):=\frac{1}{t}\big|\big\{ i\in[t]\bigmid H(\bar u_i)\neq
c_i\big\}\big|.
\]
The next theorem is a generalisation of Theorem~\ref{theo:log} where,
instead of insisting on a consistent hypothesis, we try to find a
hypothesis with minimal training error.

\begin{theorem}\label{theo:log-err}
  Let $k,\ell,q\in\NN$. Then there is a $q^*\in\NN$ and a learning algorithm $\mathfrak
  M$ for the $k$-ary learning problem over some finite background structure
  $B$ with the following properties.
  \begin{enumerate}
  \item $\mathfrak M$ always returns a hypothesis 
    $H=\llbracket\phi^*(\bar x\smid\bar v^*)\rrbracket^B$ for some $r^*$-local
    first-order formula $\phi^*(\bar x\smid\bar y)$ 
    of quantifier rank
    $q^*$ and $\bar v^*\in U(B)^\ell$.
  \item If there is a first-order formula $\phi(\bar x\smid\bar y)$ of quantifier rank 
    $q$ and some tuple $\bar v\in U(B)^\ell$ of parameters such that 
    $\err_T(\model B\phi{\bar x}{\bar v})\le\epsilon$, where $T$ is the input 
    sequence $T$, then  $\err_T(H)\le\epsilon$ for the hypothesis $H$
    returned by $\mathfrak L$ on input $T$.
  \item The algorithm runs in time $(\log n+d+t)^{O(1)}$ with only
    local access to $B$, where
    $n:=|U(B)|$ and $d:=|\Delta(B)|$ and $t:=|T|$.
  \item The hypotheses returned by $\mathfrak M$ can be evaluated in
    time $(\log n+d)^{O(1)}$ with only
    local access to $B$.
  \end{enumerate}
\end{theorem}

To prove the theorem, we use the same notation as in the previous
section: we fix $\rho,k,\ell,q$ and define $q^*,r^*$ as before. We let
$B$ be a $\rho$-structure. We define $\Phi,\Phi^*$ and $\CC,\CC^*$ as before.
We let $t\ge 1$ and $\CT:=\big(U(B)^k\times\{0,1\}\big)^t$.

The proof of Theorem~\label{theo:log-rr} relies on the following
generalisation of Lemma~\ref{lem:main}.

\begin{lemma}\label{lem:main-err}
  Let $T\in\CT$ such that $\err_T(C)\le\epsilon$ for some $C\in\CC$. 
  Then there is a formula $\phi^*(\bar x\smid\bar y)\in\Phi^*$ and a tuple
  $\bar v^*\in N_{2\ell r^*}(T)^\ell$ such that 
  \[
  \err_T\big(\model B{\phi^*}{\bar
    x}{\bar v^*}\big)\le\epsilon.
  \]
\end{lemma}

\begin{proof}
  Let $T=\big((\bar u_1,c_1),\ldots,(\bar u_t,c_t)\big)$.  Let
  $\phi(\bar x\smid\bar y)\in\Phi$ and
  $\bar v=(v_1,\ldots,v_\ell)\in U(B)^\ell$ such that
  $\err_T\big(\model B\phi{\bar x}{\bar v}\big)\le\epsilon$. Then
  there exists a subsequence 
  $S$ of $T$ such that $|S|\ge(1-\epsilon)\cdot t$ and $\model
  B\phi{\bar x}{\bar v}$ is consistent with $S$.

  By Lemma~\ref{lem:main}, there is a formula $\phi^*(\bar x\smid\bar y)\in\Phi^*$ and a tuple
  $\bar v^*\in N_{2\ell r^*}(S)^\ell$ such that $\model B{\phi^*}{\bar
    x}{\bar v^*}$ is consistent with $S$. Then $\err_T\big(\model B{\phi^*}{\bar
    x}{\bar v^*}\big)\le\epsilon$.
\end{proof}
	
\begin{figure}
  \textbf{Algorithm} $\mathfrak M$
  \begin{algorithmic}[1]
    \Require 
    Training sequence
    $T\in\CT$,\newline
    local access to background structure $B$
    \State $N\gets N_{2\ell r^*}(T)$
    \State $minerr\gets t+1$
    \ForAll{$\bar v^*\in N^\ell$}
    \ForAll{$\phi^*(\bar x\smid\bar y)\in\Phi^*$}
    \State $err\gets0$
    \ForAll{$(\bar u,c)\in T$}
    \If{$\hspace{-1em}\begin{array}[t]{c@{\,}l@{\,}}&(N_{r^*}(\bar u\bar v^*)\models\phi^*(\bar u,\bar v^*)\text{ and
                            }c=0)\\
           \text{or}&(N_{r^*}(\bar u\bar v^*) \not\models\phi^*(\bar
                      u,\bar v^*)\text{ and }c=1)\end{array}$}
                  \smallskip
    \State $err\gets\err+1$
    \EndIf
    \EndFor
    \If{$err < minerr$}
    \State$minerr\gets err$
    \State$\phi_{\min}\gets\phi^*$
    \State$\bar v_{\min}\gets\bar v^*$
    \EndIf
    \EndFor
    \EndFor
    \State\Return $\phi_{\min},\bar v_{\min}$
  \end{algorithmic}
  \caption{Learning algorithm $\mathfrak M$ of Theorem~\ref{theo:log-err}}
  \label{fig:M}
\end{figure}

\begin{proof}[Proof of Theorem~\ref{theo:log-err}]
  The pseudocode for our
  learning algorithm $\mathfrak M$ is shown in Figure~\ref{fig:L}. The
  algorithm is very similar to the algorithm $\mathfrak L$ of
  Theorem~\ref{theo:log}, except that we do not check for consistency,
  but count the errors of all hypotheses and return the one with
  minimum error. The runtime of $\mathfrak M$ is essentially the same
  as that of $\mathfrak L$.
\end{proof}

There is an analogous generalisation of Theorem~\ref{theo:uni}.

\begin{theorem}\label{theo:uni-err}
  Let $k,\ell,q\in\NN$. Then there is a $q^*\in\NN$ and
  a learning algorithm $\mathfrak
  M$ for the $k$-ary learning problem over some (possibly infinite) background structure
  $B$ with properties 1) and 2) of Theorem~\ref{theo:log-err} and the
  following two properties.
  \begin{enumerate}
  \item[3u)]The algorithm runs in time $(d+t)^{O(1)}$ under the uniform
    cost measure with only
    local access to $B$, where
    $d:=|\Delta(B)|$ and $t:=|T|$.
  \item[4u)] The hypotheses returned by $\mathfrak M$ can be evaluated in
    time $d^{O(1)}$ under the uniform cost measure with only
    local access to $B$.
  \end{enumerate}
\end{theorem}	

\begin{proof}
  The proof is the same as that of
  Theorem~\ref{theo:log-err}, except that in the analysis of the
  algorithm the log-factor disappears because of the uniform cost measure.
\end{proof}

\section{PAC Learning}
\label{sec:pac}

In this section, we sketch some of the basic principles of algorithmic
learning theory and show how they apply in our context. For more
background, we refer the reader to \cite{bluhopkan16,computationalLearningTheory,shaben14}.

So far, we have focussed on the training error of our learning
algorithms. But of course that is not the error we are mainly
interested in; our goal is to generate hypotheses that
with a low \emph{generalisation error}.  \emph{Probably approximately correct
 (PAC)} learning gives us a framework for analysing the generalisation
error theoretically.

Consider a learning problem with instance space $\BU$ where we want to
learn an unknown target concept $C^*:\BU\to\{0,1\}$. As before, we are
mainly interested in the case that $\BU=U(B)^k$ for some background
structure $B$ and that $C^*=\model B{\phi}{\bar x}{\bar v}$ for some
first-order formula $\phi(\bar x\smid\bar y)$ and parameter tuple $\bar v\in
U(B)^\ell$.
The basic assumption of PAC learning is is that there is an (unknown)
probability distribution $\CD$ on the instance space ${\BU}$ and that
instances are drawn independently from this distribution; the training
instances as well as the new instances that we want to classify with
our hypothesis. We define the \emph{generalisation error} of a
hypothesis $H$ to be the probability that $H$ is wrong on a random
instance, that is,
\begin{equation}\label{eq:err}
\err_{\CD,C^*}(H):=\Pr_{x\sim\CD}(H(x)\neq C^*(x)).
\end{equation}
We allow our algorithm to make
a small generalisation error controlled by the
\emph{error parameter $\epsilon$}. As the hypothesis $H$ depends on
the randomly chosen training examples, we must allow for an error
caused by unusually bad examples as well.  We usually quantify this
second type of error by the \emph{confidence parameter $\delta$}. Our
goal is to generate a hypothesis $H=H(T,\epsilon,\delta)$, which of
course depends on the training sequence
$T\subseteq(\BU\times\{0,1\})^t$ and may also depend on $\epsilon$ and
$\delta$, such that
\begin{equation}
  \label{eq:pac}%
  \Pr_{T\sim\CD}(\err_{\CD,C^*}(H(\epsilon,\delta,T))\le\epsilon)\ge1-\delta.
\end{equation}
Here $T\sim\CD$ indicates that the training examples are drawn
independently from $\CD$. Intuitively, a hypothesis satisfying
\eqref{eq:pac} is \emph{probably} (referring to the high confidence of at least
$1-\delta$) \emph{approximately} (referring to the low error of at most
$\epsilon$) \emph{correct}.

Ideally, we would like $\mathfrak L$ that for all target concepts
$C^*\subseteq\BU$, all probability
distributions $\CD$ on $\BU$, and all $\epsilon,\delta>0$ generates
hypotheses that are probably approximately correct. This is something we usually cannot
achieve. Thus we make assumptions about the target concept, which we
formalise by considering target concepts from a \emph{concept class}
$\CC$. We also specify the \emph{hypothesis class $\CH$} and a
function $t$ that determines the number of training examples required
by the algorithm. We say that a learning algorithm $\mathfrak L$ is a
\emph{$(\BU,\CC,\CH,t)$-PAC-learning algorithm} if for all probability distributions
$\CD$ on $\BU$, all target concepts $C^*\in\CC$, and all
$\epsilon,\delta>0$, given a sequence $T$ of $t(\epsilon,\delta)$
training examples and $\epsilon,\delta$, the algorithm generates a
hypothesis $H(T,\epsilon,\delta)\in\CH$ that satisfies \eqref{eq:pac}.

\subsection{Sample Size Bound}
It is a basic insight from computational learning theory that if the
hypothesis class $\CH$ is finite we need roughly $\log|\CH|$
training examples to achieve probable approximate correctness. The
following well-known lemma makes this precise. For a proof, see for
example \cite{shaben14}. 

\begin{lemma}[Sample Size Bound]\label{lem:ssb}
  Suppose that the hypothesis class $\CH$ is finite and that the
  length $t$ of the training sequence satisfies
  \begin{equation}
    \label{eq:6:1}
    t\ge\frac{\ln(|\CH|/\delta)}{\epsilon}. 
  \end{equation}
  Then for all probability distributions $\CD$ on $\BU$ and all target functions
  $C^*$, 
  \begin{align*}
  &\Pr_{T\sim\CD}\Big(
\err_{\CD,C^*}(H)<\epsilon\text{ for all $H\in\CH$ consistent with
    }T\big)\Big)\\
   &\ge 1-\delta.
  \end{align*}
\end{lemma}

This means that if $\CH$ is finite and the training sequence is long
enough (as specified in \eqref{eq:6:1}) then with high confidence,
every consistent hypothesis will have low generalisation error.

If we combine this lemma with Theorem~\ref{theo:log}, we obtain the
following result.

\begin{theorem}\label{theo:pac-log}
  Let $k,\ell,q\in\NN$.  Then there are $q^*,r^*,s^*\in\NN$ and
  a learning algorithm $\mathfrak
  L$ for the $k$-ary learning problem over some finite background $\rho$-structure
  $B$ with the following properties. 
  \begin{enumerate}
  \item Let $\bar x,\bar y$ be tuples of length $k,\ell$,
    respectively. Let $\CC$ be the class of all
    $\model B{\phi}{\bar x}{\bar v}$ for a $\FO[\rho]$-formula
    $\phi(\bar x\smid\bar y)$ of quantifier rank $q$ and
    $\bar v\in U(B)^\ell$, and let $\CH$ be the class of all
    $\model B{\phi^*}{\bar x}{\bar v^*}$ for a syntactically
    $r^*$-local $\FO[\rho]$-formula $\phi^*(\bar x\smid\bar y)$ of
    quantifier rank $q^*$ and $\bar v^*\in U(B)^\ell$. Let
    \[
    t(n,\epsilon,\delta)=s^*\ceil{\frac{\log (n/\delta)}{\epsilon}},
    \]
    where $n:=|B|$.

    Then $\mathfrak L$ is a $(U(B)^k,\CC,\CH,t)$-PAC-learning
    algorithm.
    \item The algorithm runs in time $(\log n+d+1/\epsilon+\log1/\delta)^{O(1)}$ with only
    local access to $B$, where
    $d:=|\Delta(B)|$.
   \end{enumerate}
\end{theorem}

\begin{proof}
  Let $\mathfrak L$ be the algorithm of Theorem~\ref{theo:log}. Given
  a training sequence $T$ of length $t$ consistent with a concept from
  $\CC$, it generates a hypothesis $H\in\CH$ consistent with $T$. It
  follows from the Lemma~\ref{lem:ssb} that $H$ is
  probably approximately correct.
\end{proof}

\subsection{VC Dimension}
We cannot apply the sample size bound in a situation where the
background structure is infinite. But there is an improved sample size
bound that even holds for (some) infinite hypothesis classes. For this
bound, we replace the factor $\ln |\CH|$ in the sample size bound by
the VC-dimension of $\CH$. The VC-dimension is a combinatorial measure
for the complexity of a set system.

Let $\BU$ be a set and $\CH\subseteq 2^{\BU}$. A set $V\subseteq\BU$
is \emph{shattered} by $\CH$ if for every $I:V\to\{0,1\}$ there is an
$H\in\CH$ such that $I$ is the restriction of $H$ to $V$.
The \emph{VC-dimension} of $\CH$, denoted by $\VC(\CH)$,  is the maximum size of a set
shattered by $\CH$, or $\infty$ if this maximum does not exist. $\CH$
has \emph{finite VC-dimension} if $\VC(\CH)<\infty$. Observe that for
finite $\CH$ we have $\VC(\CH)\le \log|\CH|$.

The following
lemma due to Blumer, Ehrenfeucht, Haussler and
Warmuth~\cite{bluehrhau+89} relates VC-dimension to PAC-learning.

\begin{lemma}[VC-Dimension Sample Size Bound \cite{bluehrhau+89}]\label{lem:vc-ssb}
  There is a constant $c$ such that the following holds. Suppose that the hypothesis class $\CH$ has finite VC-dimension and that the
  length $t$ of the training sequence satisfies
  \begin{equation}
    \label{eq:6:2}
    t\ge c\cdot\frac{\VC(\CH)+\ln 1/\delta}{\epsilon}. 
  \end{equation}
  Then for all probability distributions $\CD$ on $\BU$ and all target functions
  $C^*$, 
  \begin{align*}
  &\Pr_{T\sim\CD}\Big(
\err_{\CD,C^*}(H)<\epsilon\text{ for all $H\in\CH$ consistent with
    }T\big)\Big)\\
   &\ge 1-\delta.
  \end{align*}
\end{lemma}

We can apply this improved sample size bound in our setting for
bounded degree structures, because the VC-dimension of first-order
definable models is bounded on bounded degree structures.

\begin{lemma}[\cite{grohe2004learnability}]\label{lem:vc-bounded-deg}
  Let $d,k,\ell,q\in\NN$. Then there is an $m\in\NN$ such that the following
  holds. Let $B$ be a structure of maximum degree $\Delta(B)\le d$,
  let $\CC$ be the class of all
    $\model B{\phi}{\bar x}{\bar v}$ for a $\FO[\rho]$-formula
    $\phi(\bar x\smid\bar y)$ of quantifier rank $q$ with $|\bar
    x|=k,|\bar y|=\ell$ and
    $\bar v\in U(B)^\ell$. Then $\VC(\CC)\le m$.
\end{lemma}

Grohe and Tur\'an \cite{grohe2004learnability} bounded the
VC-dimension of first-order definable concept classes on a wide range
of further classes structures, among them planar graphs and graphs of
bounded tree width. Adler and Adler \cite{adler2010nowhere} extended this to
all nowhere dense graph classes. 

Using the VC-Dimension Sample Size Bound and the previous lemma, we
obtain the following theorem.

\begin{theorem}\label{theo:vc-uni}
  Let $d,k,\ell,q\in\NN$.  Then there are $q^*,r^*,s^*\in\NN$ and
  a learning algorithm $\mathfrak
  L$ for the $k$-ary learning problem over some (possibly infinite) background $\rho$-structure
  $B$ of maximum degree at most $d$ with the following properties. 
  \begin{enumerate}
  \item Let $\bar x,\bar y$ be tuples of length $k,\ell$,
    respectively. Let $\CC$ be the class of all
    $\model B{\phi}{\bar x}{\bar v}$ for a $\FO[\rho]$-formula
    $\phi(\bar x\smid\bar y)$ of quantifier rank $q$ and
    $\bar v\in U(B)^\ell$, and let $\CH$ be the class of all
    $\model B{\phi^*}{\bar x}{\bar v^*}$ for a syntactically
    $r^*$-local $\FO[\rho]$-formula $\phi^*(\bar x\smid\bar y)$ of
    quantifier rank $q^*$ and $\bar v^*\in U(B)^\ell$. Let
    \[
    t(\epsilon,\delta)=s^*\ceil{\frac{\log1/\delta}{\epsilon}}.
    \]
    Then $\mathfrak L$ is a $(U(B)^k,\CC,\CH,t)$-PAC-learning
    algorithm.
    \item The algorithm runs in time
      $(1/\epsilon+\log1/\delta)^{O(1)}$ under the uniform cost
      measure and with only
    local access to $B$.
   \end{enumerate}
\end{theorem}

\begin{proof}
  This follows by combining Theorem~\ref{theo:uni} with Lemmas~\ref{lem:vc-ssb}
  and \ref{lem:vc-bounded-deg}.
\end{proof}

Corollary~\ref{cor:pac} follows from this theorem. (In fact, the
theorem should be viewed as a precise version of Corollary~\ref{cor:pac}.)

\subsection{Agnostic PAC-Learning}
\emph{Agnostic learning} is a generalisation of our setting where
there is no deterministic target function (or concept), but only a
probabilistic one. In practice, this may occur in a situation where
the instances in our abstract instance space do not fully capture the
relevant properties of the real-world objects they describe. This can
easily happen, because typically the instances are tuples only
describing certain features of the objects.

We continue to consider Boolean classification problems on an instance
space $\BU$. Instead of a probability distribution on $\BU$ and a
target concept $C^*\subseteq \BU$, we now assume that we have a
probability distribution $\CD$ on $\BU\times\{0,1\}$. We define the
\emph{generalisation error} of a hypothesis $H$ to be
\[
\err_{\CD}(H):=\Pr_{(u,c)\sim\CD}(H(u)\neq c).
\]
To quantify the quality of a learning algorithm, we compare the
quality of the hypothesis of our algorithm with the best possible
hypothesis coming from a certain class $\CC$. For classes
$\CC,\CH\subseteq 2^{\BU}$ and a function $t$, we say that a learning
algorithm $\mathfrak L$ is an
\emph{agnostic $(\BU,\CC,\CH,t)$-PAC-learning algorithm} if for all probability distributions
$\CD$ on $\BU\times\{0,1\}$ and all
$\epsilon,\delta>0$, given a sequence $T$ of $t(\epsilon,\delta)$
training examples and $\epsilon,\delta$, the algorithm generates a
hypothesis $H(T,\epsilon,\delta)\in\CH$ that satisfies 
\begin{equation}
  \label{eq:apac}
  \Pr_{T\sim\CD}\big(\err_{\CD}(H)-\inf_{C\in\CC}\err_{\CD}(C)\le\epsilon\big)\ge 1-\delta.
\end{equation}
The agnostic PAC learning framework has been introduced by
Haussler~\cite{hau92}.

To obtain agnostic PAC-learning algorithms, we can use the following
Uniform Convergence Lemma instead of the Sample Size Bound of
Lemma~\ref{lem:ssb}. Recall that $\err_T(H)$ denotes the training error of
a hypothesis $H$ on a training sequence $T$.

\begin{lemma}[Uniform Convergence]\label{lem:uc}
    Suppose that the hypothesis class $\CH$ is finite and that the
  length $t$ of the training sequence satisfies
  \begin{equation}
    \label{eq:6:3}
    t\ge\ceil{\frac{2\log (2|\CH|/\delta)}{2\epsilon^2}}. 
  \end{equation}
  Then for all probability distributions $\CD$ on $\BU\times\{0,1\}$,
  \[
  \Pr_{T\sim\CD}\Big(
|\err_{\CD}(H)-\err_T(\CH)|<\epsilon\big)\Big)\ge 1-\delta.
  \]
\end{lemma}

A proof can be found in
\cite{shaben14}. 

Now from Theorem~\ref{theo:log-err} we get an agnostic PAC-learning algorithm
for first-order definable concept classes on structures of small degree.

\begin{theorem}\label{theo:pac-log-err}
  Let $k,\ell,q\in\NN$.  Then there are $q^*,r^*,s^*\in\NN$ and
  a learning algorithm $\mathfrak
  M$ for the $k$-ary learning problem over some finite background $\rho$-structure
  $B$ with the following properties. 
  \begin{enumerate}
  \item Let $\bar x,\bar y$ be tuples of length $k,\ell$,
    respectively. Let $\CC$ be the class of all
    $\model B{\phi}{\bar x}{\bar v}$ for a $\FO[\rho]$-formula
    $\phi(\bar x\smid\bar y)$ of quantifier rank $q$ and
    $\bar v\in U(B)^\ell$, and let $\CH$ be the class of all
    $\model B{\phi^*}{\bar x}{\bar v^*}$ for a syntactically
    $r^*$-local $\FO[\rho]$-formula $\phi^*(\bar x\smid\bar y)$ of
    quantifier rank $q^*$ and $\bar v^*\in U(B)^\ell$. Let
    \[
    t(n,\epsilon,\delta)=s^*\ceil{\frac{\log n/\delta}{\epsilon^2}},
    \]
    where $n:=|B|$.

    Then $\mathfrak M$ is an agnostic $(U(B)^k,\CC,\CH,t)$-PAC-learning
    algorithm.
    \item The algorithm runs in time $(\log n+d+1/\epsilon+\log1/\delta)^{O(1)}$ with only
    local access to $B$, where
    $d:=|\Delta(B)|$.
   \end{enumerate}
\end{theorem}

\begin{proof}
  Let $\mathfrak M$ be the algorithm of
  Theorem~\ref{theo:log-err}. Let $\CD$ be a probability distribution
  on $U(B)^k\times\{0,1\}$. Let $C^*\in\CC$ such that
  \[
  \err_{\CD}(C^*)=\min_{C\in\CC}\err_{\CD}(C).
  \]
  As $\CC$ is finite, this minimum exists. By the Uniform Convergence
  Lemma (Lemma~\ref{lem:uc}) applied to $\CC$ and $\epsilon/2,\delta/2$ we have
  \[
  \Pr_{T\sim\CD}\Big(
  |\err_{\CD}(C^*)-\err_T(C^*)|<\epsilon/2\big)\Big)\ge 1-\delta/2
  \]
  (assuming that the constant $s^*$ is sufficiently large).

  Given a training sequence $T$, the algorithm $\mathfrak M$ generates
  a hypothesis $H$ with at most the training error of $C^*$ on
  $T$. Applying the Uniform Convergence Lemma again, this time to
  $\CH$ and $\epsilon/2,\delta/2$, we get 
  \[
  \Pr_{T\sim\CD}\Big(
  |\err_{\CD}(H)-\err_T(H)|<\epsilon/2\big)\Big)\ge 1-\delta/2
  \]
  This implies that
  \[
   \Pr_{T\sim\CD}\Big(
  \err_{\CD}(H)-\err_{\CD}(C^*)<\epsilon\big)\Big)\ge 1-\delta.
  \]
  Hence $\mathfrak M$ is an agnostic $(U(B)^k,\CC,\CH,t)$-PAC-learning
    algorithm.
\end{proof}

\section{Conclusions}
We prove that first-order definable models are learnable in polylogarithmic
time on finite structures of polylogarithmic degree. In view of the
simple example showing that sublinear parameter learning is
impossible, which for a long time made us (or rather, the first
author) believe that sublinear learning is impossible in general in
our framework, this result came as a surprise to us.

It is less surprising that the proof relies on the locality of
first-order logic. In fact, the proof is not very difficult, but it
has to be set up in the right way. In particular, the use of syntactic
locality and the notion of local types, which to the best of our
knowledge is new, is essential. Let us remark that we cannot use
Hanf's Locality Theorem, which is usually easier to handle than
Gaifman's, because in structures of (poly)logarithmic degree the
number of isomorphism types of local neighbourhoods grows too fast.

Algorithmically, our paper is not very sophisticated: our algorithms
are simple brute-force algorithms that are practically useless due to
enormous hidden constants. It is a very interesting question whether
there are also practical algorithms for learning \FO-definable models
(which may not be more efficient than ours in the worst case, but nevertheless
work better). One approach would be to map the data (consisting of tuples of
elements and the background structure) to high dimensional feature
vectors, maybe even obtain a kernel, and then apply conventional machine
learning algorithms.

As we have outlined in the introduction, our results may be viewed as
a contribution to a descriptive complexity theory of machine
learning. Within such a theory, many questions, both technical and
conceptual, remain open. For example, are there sublinear learning
algorithms for first-order logic on other classes of structures such
as words, trees, planar graphs? At a more fundamental level, what is a
good computation model (replacing our ``local access'' model) on
background structures that are still sparse, but have large maximum
degree. In fact, sublinear algorithms seem unlikely for most classes
of high maximum degree. But what about fixed-parameter tractable
learning algorithms.  As soon as we allow formulas with arbitrarily
many instance and parameter variables (that is, allow unbounded $k$
and $\ell$), fixed-parameter tractability becomes nontrivial on any of
the classes suggested above.
And what about other logics, for example monadic second-order
logic or modal and temporal logics. Finally, one should generalise the
framework from Boolean classification problems to other types of
learning problems.

If our version of a declarative approach to
machine learning is supposed to have any impact in practice, maybe the
most important question is: what are suitable logics and background
structures for expressing relevant and feasible machine learning
models?

\section*{Acknowledgements}
We would like to thank Kristian Kersting and Daniel Neider for very
helpful comments on an earlier version of this paper.

\end{document}